\documentclass[final]{l4dc2024}

\title[Understanding the Difficulty of Solving Cauchy Problems with PINNs]{Understanding the Difficulty of Solving Cauchy Problems with PINNs}
\usepackage{times}
\usepackage{caption}
\usepackage{subcaption}
\usepackage{wrapfig}
\usepackage{todonotes}

\renewcommand{\thefootnote}{\rfnsymbol{footnote}}

\coltauthor{\Name{Tao Wang} 
\Email{taw003@ucsd.edu}\\
  \Name{Bo Zhao} 
  \Email{bozhao@ucsd.edu}\\
  \Name{Sicun Gao} \Email{sicung@ucsd.edu}\\
  \Name{Rose Yu} \Email{roseyu@ucsd.edu}\\
  \addr University of California, San Diego}

\begin{document}

\maketitle

\begin{abstract}%
Physics-Informed Neural Networks (PINNs) have gained popularity in scientific computing in recent years. However, they often fail to achieve the same level of accuracy as classical methods in solving differential equations. In this paper, we identify two sources of this issue in the case of Cauchy problems: the use of $L^2$ residuals as objective functions and the approximation gap of neural networks. We show that minimizing the sum of $L^2$ residual  and initial condition error is not sufficient to guarantee the true solution, as this loss function does not capture the underlying dynamics. Additionally, neural networks are not capable of capturing singularities in the solutions due to the non-compactness of their image sets. This, in turn, influences the existence of global minima and the regularity of the network. We demonstrate that when the global minimum does not exist, machine precision becomes the predominant source of achievable error in practice. We also present numerical experiments in support of our theoretical claims.
\end{abstract}

\renewcommand{\thefootnote}{\arabic{footnote}}

\section{Introduction}

Neural networks have attracted significant attention in the learning of dynamical systems \citep{ruiwang, djeumou}. One example is  solving partial differential equations (PDEs) with Physics-Informed Neural Networks (PINNs) \citep{sirignano2018dgm, raissi2019physics}, which are traditionally handled by classical methods such as finite element method (FEM) \citep{ern}. Despite the promise, several works have identified failure modes of PINNs. It has been observed that when solving one-dimensional Burgers equations, classical FEM can achieve the $L^2$ error \footnote{In some work, the mean-squared error (MSE), or the square of the standard $L^2$ error, is used as the measure of approximation error, usually resulting in lower values. Here we convert it to the corresponding $L^2$ error.} on the magnitude of $10^{-7}$ \citep{khater}, in contrast to the error of magnitude $10^{-2}$ attained by PINNs \citep{krishnapriyan2021characterizing} or DeepONet \citep{lu2022multifidelity}.

In this paper, we develop a fundamental understanding of the failure mode of PINNs. We focus on Cauchy problems, a class of important PDEs for evolution equations in optimal control \citep{fleming} and fluid dynamics \citep{sell}.  Our analysis reveals two main aspects of the failure mode: the use of $L^2$ residuals  as the loss function and neural networks as function approximators. On the theoretical side, achieving zero loss on $L^2$ residual and initial/boundary error does not guarantee the true solution \citep{courant, evans}. PINNs solve a PDE over a given compact domain. In Section~\ref{sec:residual}, we demonstrate that such a setting may conflict with the underlying evolutionary dynamics, such as the propagation of characteristics in first-order equations and the non-local behavior of parabolic equations. For instance, when solving the Burgers' equation, minimizing $L^2$ residual alone produces a smooth solution, whereas the true solution exhibits high-frequency behavior as in Figure~\ref{fig:intro}. In contrast, traditional methods do not encounter this issue because they "respect" the underlying dynamics and leverage it to solve equations through the method of characteristics and/or Fourier transform.

\begin{wrapfigure}{r}{0.62\textwidth}
\centering
\includegraphics[width=0.3\textwidth]{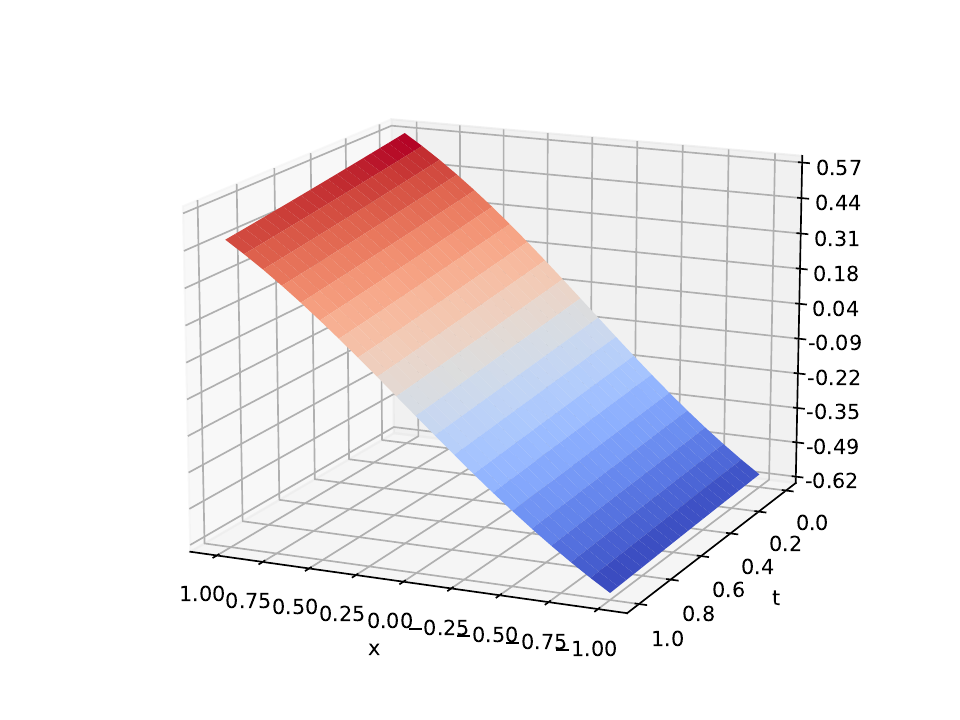}
\includegraphics[width=0.3\textwidth]{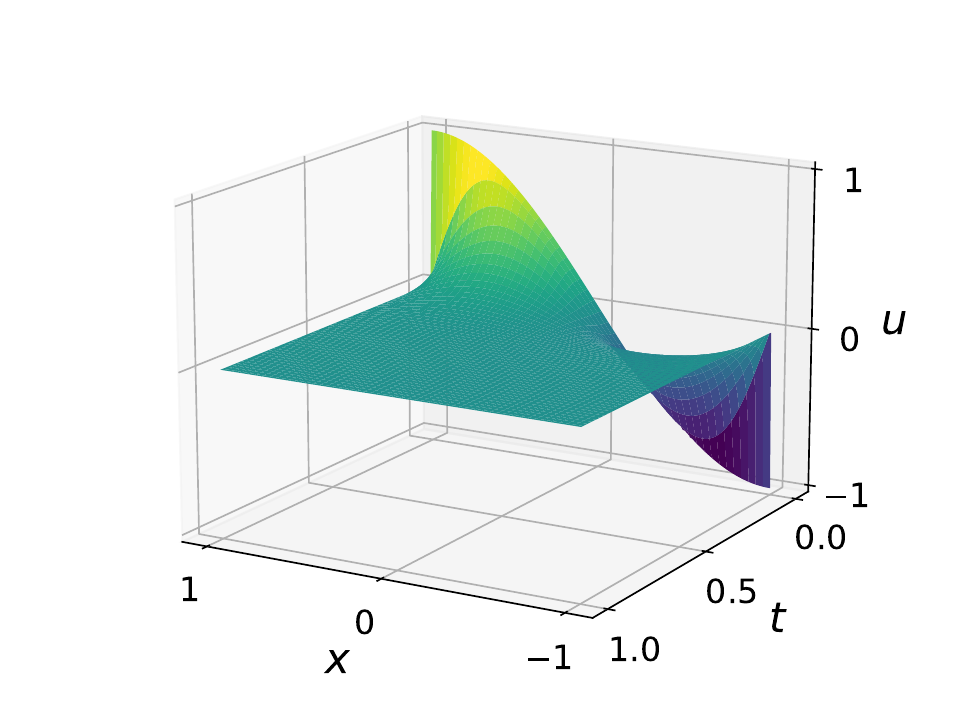}
\caption{An example where PINNs fail to solve (left), compared to the true solution (right).}
\label{fig:intro}
\end{wrapfigure}

Regarding the approximation power of neural networks, the universal approximation theorem (UAT) states that a two-layer neural network with sigmoid activation can approximate \textit{any} continuous functions if it is wide enough \citep{funahashi, hornik, barron}. However, the achievable error in practice is usually much higher than the one UAT suggests. In Section~\ref{sec:nnappro}, we show that the intersection of the image set of a neural network and the closed unit ball is \textit{not} compact in function space, which implies that having a global minimum at infinity is possible, especially when the target function has discontinuity. The influence of machine precision is significant in this case due to the exponential decay of gradients in activation functions such as \textit{sigmoid} and \textit{tanh}. Our new lower bound on actual error explains why neural networks cannot achieve an approximation error below certain thresholds in solving Cauchy problems that have discontinuous solutions.

Our main contribution in this paper can be summarized as follows:

\begin{itemize}

  \item We investigate the difficulty of solving Cauchy problems with PINNs from two perspectives: the loss function of $L^2$ residuals minimization and the approximation gap of neural networks.
  
  \item For the learning objective, we demonstrate that solving Cauchy problems over compact domain can be ill-posed and poorly formulated due to its incapability of capturing the underlying dynamics. Error estimates are obtained for both first-order and second-order equations.

  \item Regarding the approximation power of neural networks, we show that there might be no global minimum when approximating discontinuous functions. This will, in turn, result in a higher error in solutions due to machine precision.

  \item Our results suggest two things for PINNs: first, loss functions should be designed in accordance with the evolutionary dynamics within PDEs, for instance, including regularity conditions and proper boundary values; second, when the true solution has discontinuity, neural networks are not capable of representing it accurately.

\end{itemize}


\section{Related Work}

\paragraph{Limitation of ML methods in solving PDEs.} Solving partial differential equations (PDEs) is one of the core areas in scientific computing. A number of deep learning algorithms have been developed to learn PDE solutions, such as the deep Galerkin method \citep{sirignano2018dgm}, physics-informed neural networks (PINN) \citep{raissi2019physics}, Fourier neural operator \citep{kovachki2021neural}, and DeepONet \citep{lu2019deeponet, lu2022multifidelity}.
Although some of these methods are proven to have the universal approximation property \citep{lu2019deeponet,kovachki2021universal}, their performances in practice are often not compared to traditional methods such as FEM, whose error provably approaches zero as the resolution increases.
In particular, PINN has worse accuracy compared to traditional computational fluid dynamics methods \citep{cai2021physics} and can fail to produce a physical solution in certain problems \citep{chuang2022experience, wang2022}.
This limitation may be attributed to spectral bias and convergence rates of different loss components \citep{wang2022and}. Even when the neural network possesses enough expressiveness to represent the objective function, PINN can fail to obtain the true solution, which motivates analyzing the loss landscape \citep{fuks, mcclenny, krishnapriyan2021characterizing, zhao2022}. Our work seeks to understand this limitation by focusing on the conflict between the evolutionary structure of Cauchy problems and the formulation of PINNs.

\paragraph{The approximation power of neural networks.} Since the space of square-integrable functions over $D$, say $L^2(D)$, is separable \citep{folland}, every basis can perfectly represent any function in $L^2(D)$ if it is allowed to have infinitely many parameters. Thus, the asymptotic guarantee of neural network UAT cannot distinguish neural networks from existing function approximators. Inspired by this fact, many attempts have been made to understand the approximation power of neural networks by evaluating their performance under a fixed number of parameters \citep{berner, devore, petersen}. In particular, the approximation power of neural networks is demonstrated in the way that they can achieve the same level of error as those by classical basis functions, such as polynomials and trigonometric functions, but with fewer parameters \citep{liang, elbrachter, kim}. The error analysis is usually done under the assumption that the target function is Lipschitz continuous. However, most of these frameworks do not apply to non-smooth and discontinuous solutions which is common in scientific computing.

\section{The Difficulty of Solving Cauchy Problems with PINN Loss}
\label{sec:residual}
In this section, we investigate why it is difficult for PINNs to obtain accurate solutions in Cauchy problems from the loss function perspective. We consider the Cauchy problem for first-order and second-order parabolic equations separately, and analyze how PINNs lead to failures in these cases.

\subsection{Problem formulation}

Consider the general form of Cauchy problems:
\begin{equation}
\label{eq:pde}
    \begin{cases}
        & u_t + \mathcal{F}(x, t, u, \nabla u, \Delta u) = 0, \quad (x, t) \in  \mathbb{R}^d \times [0, T];\\
      & u(x, 0) = \phi(x), \quad x \in \mathbb{R}^d;  \\
    \end{cases} 
\end{equation}
where $\mathcal{F}(\cdot)$ is the differential operator representing the PDE model, $u(x, t) \in \mathbb{R}^d \times [0, T]$ is the solution, $(x, t)$ is the space-time pair, $[0, T]$ is the time domain and $\phi(\cdot) \in \mathbb{R}^d$ is the initial condition. Under the framework of PINNs, the solution $u(x, t) = u(x, t; \theta)$ is represented by a neural network parameterized by $\theta \in \mathbb{R}^N$. The neural network is trained to minimize the following loss function, which combines the $L^2$ residual of the PDE model and the fitting error on the initial condition:
\begin{equation}
\label{pdeloss}
    J(\theta) = \| u_t(\cdot; \theta) + \mathcal{F}(u(\cdot; \theta))  \|^2_{L^2(U \times [0, T])} + \| u(\cdot, 0; \theta) - \phi(\cdot) \|^2_{L^2(U )}.
\end{equation}
In practice, PINNs are only able to solve PDEs over compact domains. Therefore, $U$ in \eqref{pdeloss} is usually a compact set instead of $\mathbb{R}^d$. As we will show next, this limitation may prevent PINNs from obtaining the true solutions.

For the remainder of this section, our focus is on the $L^2$ error between the optimal solution $v \in C^2(U \times [0, T])$ obtained by PINNs, which achieves zero loss in \eqref{pdeloss}, and the true solution $u(\cdot)$. This assumption implies that $v_t + \mathcal{F}(x, t, v, \nabla v, \Delta v) = 0$ for all $(x, t) \in U \times [0, T]$ and $v(x, 0) = \phi(x)$ for all $x \in U$. 

\subsection{First-order PDEs that have characteristics}

Generally, the solution of a Cauchy problem is determined by three factors: the differential operator,  initial conditions and regularity conditions. While many PDEs do not have solutions in the classical sense, others possess infinitely many qualified solutions, among which only one represents the real physical solution. Undesired solutions are eliminated by introducing regularity conditions. Therefore, minimizing the PINN objective in \eqref{pdeloss} itself is insufficient to determine the true solution as it only reflects the differential operator and initial condition without any regularity conditions. 

\paragraph{Non-uniqueness of the solution.}

Although regularity conditions assume some level of smoothness on the true solution, smoothness alone does not guarantee uniqueness. According to the Rademacher's theorem, a locally Lipschitz continuous function is differentiable \emph{almost everywhere}. However, the following theorem shows that Lipschitz continuity is not sufficient to determine the uniqueness of the solution:

\begin{theorem}
    There exists a first-order PDE with proper initial/boundary conditions that has infinitely many Lipschitz continuous solutions.
\end{theorem}

\begin{proof}
    Consider the first-order Hamilton-Jacobi equation
\begin{equation}
\label{eq:hyper2}
    \begin{cases}
        & u_t + |u_x|^2 = 0, \quad (x, t) \in  \mathbb{R} \times (0, \infty);\\
      & u(x, 0) = 0, \quad x \in \mathbb{R}.  \\
    \end{cases} 
\end{equation}
Obviously, $u_0(x, t) \equiv 0$ is a solution to \eqref{eq:hyper2} and also the most intuitive one. However, the function $u_a(x, t) = \begin{cases}
        & 0, \quad |x| \geq a t\\
      & ax - a^2 t, \quad 0 \leq x \leq a t  \\
      & -ax - a^2 t, \quad -at \leq x \leq 0  \\
\end{cases}$ is a Lipschitz continuous solution as well for all $a > 0$. As every $u_a$ solves \eqref{eq:hyper2} a.e., they all achieve zero loss under the objective function \eqref{pdeloss}.
\end{proof}

\paragraph{The overflow of characteristics.}

One of the most important properties of first-order PDEs is that their solutions are determined by the propagation of characteristics, a principle that also applies to the solution obtained by minimizing the PINN loss.  For instance, consider the one-dimensional transport equation:
\begin{equation}
\label{eq:transport}
    \begin{cases}
        & u_t + b(x) \ u_x = c(x), \quad (x, t) \in  \mathbb{R} \times (0, \infty);\\
      & u(x, 0) = \phi(x), \quad x \in \mathbb{R}.  \\
    \end{cases} 
\end{equation}
where $\phi \in C^2(\mathbb{R})$ is the initial condition, and $b(\cdot), c(\cdot)$ are sufficiently smooth. Let $D_x \subset \mathbb{R}^d$ be a compact set and $D = D_x \times [0, T]$, then we have the following estimation of the $L^2$ error:
\begin{theorem}
\label{th:first}
    Let $v \in C^1(D)$ be the solution that achieves zero $L^2$ residual and initial condition error over $D$, and $u \in C^1(\mathbb{R} \times [0, T])$ be the true solution of \eqref{eq:transport} over the entire strip $\mathbb{R} \times [0, T]$, then the $L^2$ error $\| v - u \|_{L^2(D)}$ is determined by the error along the boundary $\partial D$.
\end{theorem}

\begin{proof}
    To see how it is derived from characteristics, consider the corresponding ODE $\Dot{y} = b(y)$, with initial state $y(0) = y_0 \in \mathbb{R}$. Let $V(t) = u(y(t), t)$ and $u \in C^1(\mathbb{R} \times [0, T])$ be a function that satisfies \eqref{eq:transport} for some $T > 0$. Then, we have
$$\Dot{V}(t) = u_t(y(t), t) + u_x(y(t), t) \cdot \Dot{y} = u_t(y(t), t) + b(y(t)) \ u_x(y(t), t) = c(y(t))$$
for all $t \geq 0$. Integrating from $0$ to $t$ and applying the initial condition $V(0) = \phi(y_0)$ yields $u(y(t), t) = \phi(y_0) +  \int_0^t c(y(s)) \ \mathrm{d}s.$ Therefore, the solution is uniquely determined along the characteristic curve $y(\cdot)$. However, the PINN loss considers a compact domain $D_x \subset \mathbb{R}$ and an arbitrary point $(x, t) \in D_x \times [0, T]$, the characteristics passing through $(x, t)$ is not guaranteed to stay in $D$ before reaching the axis $\mathbb{R} \times \{t = 0\}$ when going backward in time. In this case, we have $u(y(t), t) = u(y_{t_0}, t_0) +  \int_{t_0}^t c(y(s)) \ \mathrm{d}s,$ where $y_{t_0} \in \partial D_x$ is on the boundary with $0 < t_0  \leq t$. Similarly, we have $v(y(t), t) = v(y_{t_0}, t_0) +  \int_{t_0}^t c(y(s)) \ \mathrm{d}s,$ where $y(\cdot)$ is the same characteristics as previously (illustrated in Figure~\ref{fig:character}). 

Therefore, the error between $u$ and $v$ at any interior point $(x, t)$ is equal to the error at the initial point of the corresponding characteristics, which is 
$$|u(y(t), t) - v(y(t), t)| = |u(y_{t_0}, t_0) - v(y_{t_0}, t_0)|$$
where $y(t) = x$ and $y_{t_0} \in \partial D_x$ if $t_0 > 0$. 

Hence, the $L^2$ error in the solution is determined by the error along the boundary of the domain ($\partial D_x \times [0, T]$).
\end{proof}

\begin{wrapfigure}{r}{0.35\textwidth}
\centering
\includegraphics[width=0.35\textwidth]{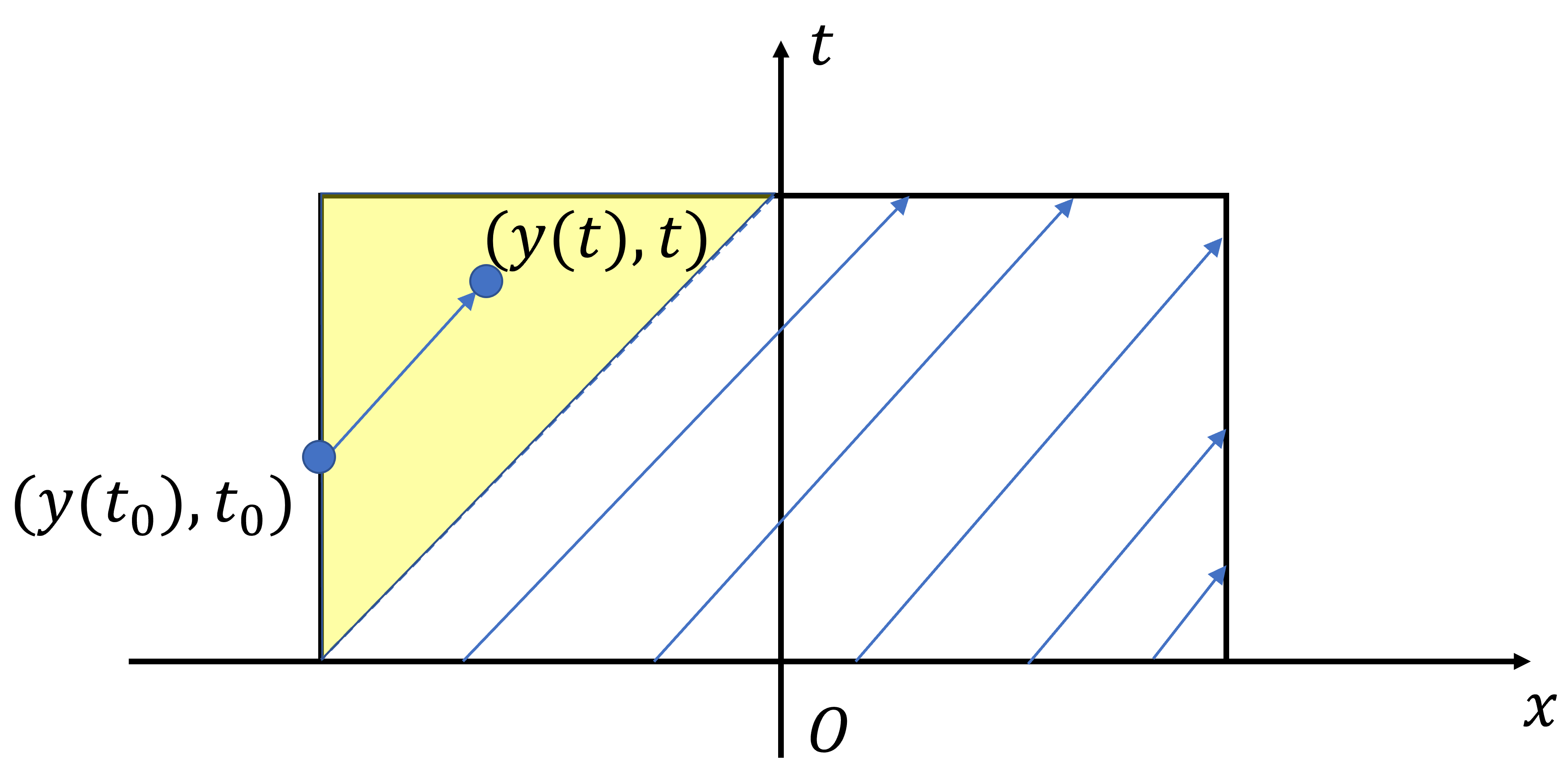}
\caption{The shaded region is not covered by valid characteristics.}
\label{fig:character}
\end{wrapfigure}

To better illustrate the claim of Theorem~\ref{th:first}, let's consider a simple example of solving \eqref{eq:transport} over the compact domain $(x, t) \in D = [-1, 1] \times [0, 1]$ in the case of $b(x) \equiv 1$, $c(x) \equiv 0$ and $T = 1$. In this case, for any point inside the region $\{ (x, t): -1 \leq x \leq 0, x + 1 \leq t \leq 1 \}$, the backward characteristics that passes through it will hit the boundary $\{ -1 \} \times [0, 1]$ first before reaching $\{ t=0 \}$ as in Figure~\ref{fig:character}, which leads to the following analytical formula for the $L^2$ error:
\begin{corollary}
    Let $v \in C^1(D)$ be the solution that achieves zero $L^2$ residual and initial error over $D$, and $u \in C^1(\mathbb{R} \times [0, T])$ be the true solution of \eqref{eq:transport} over the entire strip $\mathbb{R} \times [0, T]$ with $b \equiv 1$ and $c \equiv 0$, then the $L^2$ error is given by
    \begin{equation}
    \label{tperror}
        \| v - u \|_{L^2(D)} = \sqrt{\int_0^1 \int_{0}^t |v(-1, t - x) - u(-1, t - x)|^2 \ \mathrm{d}x \mathrm{d}t}, 
    \end{equation}
    where the term on the right-hand side depends only on the boundary values of $u$ and $v$.
\end{corollary}



Since we do not have access to the correct value on the boundary $\{ -1 \} \times [0, 1]$ beforehand, the error estimate \eqref{tperror} depends on the initialization of network parameters and hence cannot be controlled. In contrast, traditional methods solve \eqref{eq:transport} along the propagation of characteristics instead of over a pre-defined compact domain, thereby providing accurate solutions.

\subsection{Second-order parabolic equations with non-local solutions}

For parabolic equations, the $L^2$ error between a solution that achieves zero loss in \eqref{pdeloss} and the true solution can be arbitrarily large due to the non-local behavior of the solution. To make it precise, consider the $d$-dimensional parabolic equation:
\begin{equation}
\label{eq:heat}
    \begin{cases}
        & u_t + \mathcal{F}(x, t, u, \nabla u, \Delta u) = 0, \quad (x, t) \in  \mathbb{R}^d \times (0, T];\\
      & u(x, 0) = \phi(x), \quad x \in \mathbb{R}^d;  \\
    \end{cases} 
\end{equation}
where $\mathcal{F}(x, t, u, \nabla u, \Delta u) = -\sum_{i, j =1}^d a_{ij}(x, t) u_{x_i x_j} + \sum_{i = 1}^d b_i(x, t) u_{x_i} + c(x, t) u$ and the coefficients $a_{ij}(\cdot, \cdot), b_i(\cdot, \cdot), c(\cdot, \cdot)$ are sufficiently smooth and $\phi \in C^2(\mathbb{R})$ is the initial condition. Also, we assume that there exists $\gamma > 0$ such that $\sum_{i, j =1}^d a_{ij}(x, t) \xi_i \xi_j \geq \gamma |\xi|^2$ for all $x \in  U $ and $\xi = [\xi_1, ..., \xi_d] \in \mathbb{R}^d$ so that the matrix $(a_{ij}(x, t))_{i, j}$ is always positive definite. Under these assumptions, the solution of \eqref{eq:heat} can be written in the following form:
\begin{equation}
    u(x, t) = \int_{\mathbb{R}^d} \Gamma(x, t; \xi, \tau) \phi(\xi) \ \mathrm{d}\xi, \quad (x, t) \in \mathbb{R}^d \times [0, T],
\end{equation}
where $\Gamma(x, t; \xi, \tau)$ is called the fundamental solution of the parabolic equation whose explicit form can be found in the Lemma 2 of \cite{dressel}. Note that when \eqref{eq:heat} is the one-dimensional heat equation $u_t = u_{xx}$, the fundamental solution $\Gamma(x, t; \xi, \tau) = \frac{1}{2\sqrt{\pi t}} e^{\frac{|x - \xi|^2}{4t}}$ is exactly the heat kernel.

Let $D = D_x \times [0, T]$, with compact subset $D_x \subset \mathbb{R}^d$, denote the compact domain over which we want to solve \eqref{eq:heat}. For each solution candidate $\eta \in C^2(D)$, the objective functional is given by 
$$P(\eta) = \| \eta_t + \mathcal{F}(x, t, \eta, \nabla \eta, \Delta \eta)  \|^2_{L^2(D)} + \| \eta - \phi \|^2_{L^2(D_x)}.$$ Suppose that $\phi_2 \in C^2(\mathbb{R}^d)$ is another initial condition that has $\phi_2(x) = \phi(x)$ for all $x \in D_x$, and let $v(x, t) = \int_{\mathbb{R}^d} \Gamma(x, t; \xi, \tau) \phi_2(\xi) \ \mathrm{d}\xi$. It can be verified that $v$ solves \eqref{eq:heat} over $D$, i.e., $P(v) = 0$. However, $\phi_2$ is arbitrary outside $D_x$ and the kernel $\Gamma(x, t; \xi, \tau)$ is positive for all $x \in \mathbb{R}^d$, which means that we can choose $\phi_2$ such that the resulting error is arbitrarily large. Therefore, we have proved the following result:
\begin{theorem}
    For any $K > 0$ and compact set $D_x \subset \mathbb{R}^d$, there exists a function $v \in C^2(D_x \times [0, T])$ that achieves zero $L^2$ residual and initial condition error over $D = D_x \times [0, T]$, such that $\| v - u \|_{L^2(D)} > K$ where $u(\cdot)$ is the true solution of \eqref{eq:heat}.
\end{theorem}

In summary, the non-local property of parabolic equations, induced by the kernel function $\Gamma(x, t; \xi, \tau)$, makes the PINN formulation of the Cauchy problems ill-posed. This leads to the difficulty in learning the true solution over compact domains, since even achieving zero training loss does not guarantee a small $L^2$ error.

\section{The Approximation Gap of Neural Networks}
\label{sec:nnappro}
Aside from the limitations of the PINN loss, the structure of neural networks can also be a problem when representing the solution, especially when it is discontinuous. In this section, we focus on the special structure of their image sets and how it affects the accuracy through machine precision. Let $f: \mathbb{R}^N \rightarrow L^2(D)$ denote a neural network, $D \subset \mathbb{R}^d$ be a compact set and $\mathbb{R}^N$ the parameter space.

\subsection{The topology of image set $Im(f)$}
\label{sec:topology}

Now we study the complexity of the image set $Im(f) \subset L^2(D)$. Since unbounded sets are never compact in normed spaces, we will focus on their intersections with the closed unit ball $\mathcal{B}$ in $L^2(D)$, namely $\overline{Im(f)} \cap \mathcal{B}$, to better describe the complexity of the closure of image set $\overline{Im(f)}$. Every closed and bounded set in finite-dimensional Euclidean spaces is compact, but closed and bounded sets may not be compact when the underlying space is infinite-dimensional, such as in $L^2$ spaces. Therefore, compactness is no longer a trivial implication of boundedness plus closedness. To get a better sense of what compact sets in $L^2(D)$ look like, the following theorem provides a full characterization of when a bounded set has compact closure in function spaces:






\begin{proposition}
\label{prop}
    (Fréchet-Kolmogorov theorem, \cite{brezis}) Let $1 \leq p < \infty$ and $D \subset \mathbb{R}^d$ be a bounded measurable set. Then, for any bounded set $S \subset L^p(D)$, its closure $\bar{S}$ is compact if and only if for any $\epsilon > 0$, there exists $\delta > 0$ such that $\|  f(x + h) - f(x) \|_p < \epsilon$ for all $f \in S$ and all $h \in \mathbb{R}^d$ with $| h | < \delta$.
\end{proposition}

The above condition is equivalent to the statement that the functions in $S$ have $\lim_{| h | \rightarrow 0}\|  f(x + h) - f(x) \|_p = 0 $ uniformly, indicating that $S$ cannot contain ``too many'' discontinuous functions. For instance,  consider the set $S = \{g_i\}_{i=1}^\infty \subset L^2([0, 1])$ where $g_i(x) = \sqrt{i} \chi_{[0, \frac{1}{i}]}(x)$, then for any $\delta > 0$, there exists $i_0 > 2 \delta^{-1}$ such that $\| g_{i_0}(x + \frac{\delta}{2}) - g_{i_0}(x) \|_2 = \sqrt{2}$, which implies that $S$ is not compact according to Proposition~\ref{prop}.

Next, we consider the image set of a neural network whose activation function is either \textit{sigmoid} or \textit{ReLU}. Neural networks may converge to step functions as the corresponding weights approach infinity as in Figure~\ref{fig:step}, which leads to the following result:

\begin{wrapfigure}{r}{0.35\textwidth}

\vspace{-12pt}
\centering
\includegraphics[width=0.33\textwidth]{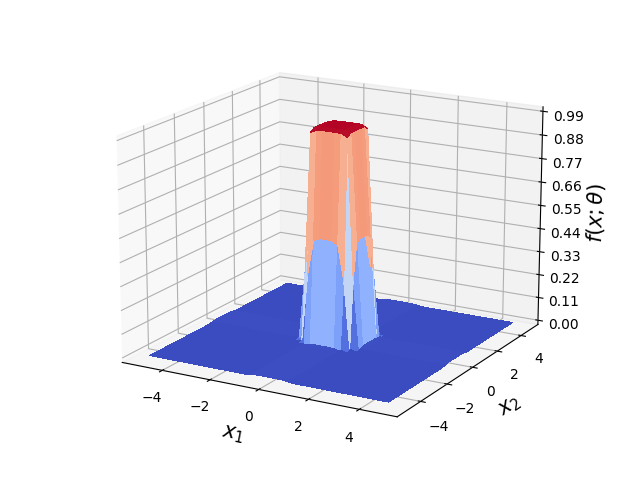}
\caption{Neural networks can converge to step functions.}

\label{fig:step}
\end{wrapfigure}



\begin{theorem}
\label{th:main}
    Let $f: \mathbb{R}^N \rightarrow L^2(D)$ be a neural network with sigmoid or \textit{ReLU} activation function, having a width greater than or equal to $2$ and at least $2$ hidden layers. Then, the set $Im(f) \cap \mathcal{B}$ is not compact. 
\end{theorem}

\begin{proof}
This proof consists of two parts: first, we prove the claim that \textit{sigmoid} and \textit{Relu} networks can approximate step functions; second, we will show that the set of normalized step functions is non-compact. 

    \textit{Sigmoid:} Let $\sigma$ be \textit{sigmoid}. For any hypercube $R = [l_1, r_1] \times ... \times [l_d, r_d]$, consider the network
    $$f(x; \theta_n) = \sigma(n y - n (d + \frac{1}{2})), $$
    $$\quad y = \sum_{i = 1}^d (\sigma(w^T_{i, 1} x + b_{i, 1}) + \sigma(w^T_{i, 2} x + b_{i, 2}))$$
    where $w_{i, 1} = [0, ..., 0, n, 0, ..., 0]^T$ and $w_{i, 2} = [0, ..., 0, -n, 0, ..., 0]^T$ are aligned with the $i$-th axis, $b_{i, 1} = -n l_i$ and $b_{i, 2} = n r_i$. It can be verified that the limit of the sequence $\{ f(x; \theta_n) \}_{n=1}^\infty$ exists for almost every $x \in \mathbb{R}^d$, which is exactly the indicator function $\chi_D$. Also, since $f(x; \theta_n) \in Im(f)$ for all $n$, we have $\chi_D \in \overline{Im(f)}$. This can be easily extended to the case of $K$ sums.

    \textit{Relu:} Without loss of generality, it suffices to show that the indicator function of interval $[0, 1]$ is approximated by a sequence of \textit{Relu} networks $\{ f(x; \theta_n) \}$ that are given by
    $$f(x; \theta_n) = \sigma(y - 2 + \frac{1}{n}), \quad y = \sigma(n x + \frac{1}{2}) - \sigma(n x - \frac{1}{2}) + \sigma(-n x + n - \frac{1}{2})$$
    so that $f(x; \theta_n) \rightarrow \chi_{[0, 1]}$ a.e. $x \in [-1, 1]$ as $n \rightarrow \infty$. This can be extended toward higher-dimensional cases, and thus we omit the detail.

    Next, let us prove the non-compactness of the set of normalized step functions. Consider the function sequence $\{ \phi_n \}_{n=1}^\infty \subset \mathcal{B}$ where $\phi_n$ is given by 
    \begin{equation*}
        \phi_n = 
        \begin{cases}
            & 2^{2n}, \quad x \in  (1 - 2^{1-n}, 1 - 2^{-n});\\
      & 0, \quad \textit{elsewhere}.
        \end{cases} 
    \end{equation*}
Then we have $\| \phi_m - \phi_n  \|_2 = \sqrt{2}$, which implies that $\{ \phi_n \}$ has no convergent subsequence. Thus, $\{ \phi_n \}_{n=1}^\infty$ cannot be sequentially compact. Since compactness and sequential compactness are equivalent in metric spaces, the set of normalized step functions is not compact as it contains $\{ \phi_n \}_{n=1}^\infty$.
\end{proof}

The non-compactness of neural network image sets can lead to the following two problems:
\begin{itemize}
    \item \emph{Global minimum at infinity}: Usually, if there exists a global minimum, it has a region of attraction in its neighborhood in which the objective function is strictly convex and guarantees local convergence. When there is no global minimum, however, the minimal achievable error is much higher as discussed in \ref{sec:global}. 

    \item \emph{Loss of regularity}: While neural networks represent Lipschitz continuous functions with bounded parameters, they may lose smoothness as weights approach infinity. According to Theorem~\ref{th:main}, there could be a sequence of parameters $\{ \theta_i \} \subset \mathbb{R}^N$ that produces a decreasing sequence $\{ J(\theta_i) \}$ but diverges to infinity itself. While each neural network $f(\theta_i)$ is Lipschitz continuous for all $i \in \mathbb{N}$, the continuity of $\lim_{i \rightarrow \infty} f(\theta_i)$ is not always implied. Therefore, regularity cannot be guaranteed through training, even if the training loss continues to decrease.

\end{itemize}

\subsection{Unreachable Infinity}
\label{sec:global}

Having the global minimum at infinity would not affect trainability if we could approach it. Unfortunately, this is not the case due to machine precision, since the weights in a neural network must tend towards infinity, which is prohibited by exponentially-decaying gradients for activation functions like  \textit{sigmoid} or \textit{tanh}. In particular, we have the following estimate:

\begin{theorem}
\label{th:minimum}
   (Lower bound on actual error) Let $f$ be a neural network with \textit{sigmoid} or \textit{tanh} activation functions and $\phi$ be the target function. Suppose that the target function $\phi$ has discontinuity, then the smallest attainable $L^2$ error $e \sim \mathcal{O}( \sqrt{\Delta x} \ |\log \epsilon|^{-\frac{1}{2}})$ where $\Delta x < 1$ is the grid size and $\epsilon > 0$ is the machine precision.
\end{theorem}

\begin{proof}
    Consider a finite dataset $X = \{ x_1, ..., x_N \} \subset \mathbb{R}^d$, and  a neuron $f(x; \theta) = \sigma(w^T x + b)$ where $\sigma$ is the \textit{sigmoid} activation function. Since the precision of a computing machine is always finite, let $\epsilon = 2^{-p}$ be the precision where $p \in \mathbb{N}$. Then for any quantity $\eta \in \mathbb{R}$ with $|\eta| < \epsilon$, it yields $\eta = 0$ on the machine. Let $\phi(\cdot)$ be the target function with $\max_{x \in D}|\phi(x)| \leq M$. For each $x \in D$, the norm of the gradient of approximation error $L(x, \theta) = (f(x; \theta) - \phi(x))^2$ with respect to $w$ is 
$$|\nabla_w L (x, \theta)| = 2|f(x; \theta) - \phi(x)| \ |\frac{2x e^{-(w^Tx + b)}}{(1 + e^{-(w^Tx + b)})^2}| \leq 2 (M + 1) M' e^{-|w^Tx + b|}$$
when $M' = \max_{x \in D} | x |$. Similarly, we have $|\nabla_b L (x, \theta)| \leq 2 (M + 1) e^{-|w^Tx + b|}$. Therefore, $\nabla_\theta L(x, \theta)$ is zero on the machine when both quantities are less than $\epsilon$, i.e., there exists 
$$\delta = \max \{ (p-1)\log 2 - \log((M + 1) M'), (p-1)\log 2 - \log(M + 1), 0 \}$$
such that $\nabla_\theta L (x_i, \theta) = 0$ for all data point $x_i$ outside the region $\mathcal{E}(w, b) = \{ x \in \mathbb{R}^d: |w^T x + b| \geq  \delta \}$. In other words, such points have no contribution to the gradient of the loss objective with respect to the parameter of this neuron. Notice that $\mathcal{E}(w, b)$ is exactly the collection of all points whose distance to the hyperplane $w^T x + b = 0$ is $\frac{\delta}{|w|}$ where the constant $\delta$ depends only on the precision and the target function, the width of effective region $\mathcal{E}(w, b)$ decays as $\|w\|$ grows at the rate of $\mathcal{O}(\|w\|^{-1})$.

\begin{wrapfigure}{r}{0.25\textwidth}
\centering
\includegraphics[width=0.25\textwidth]{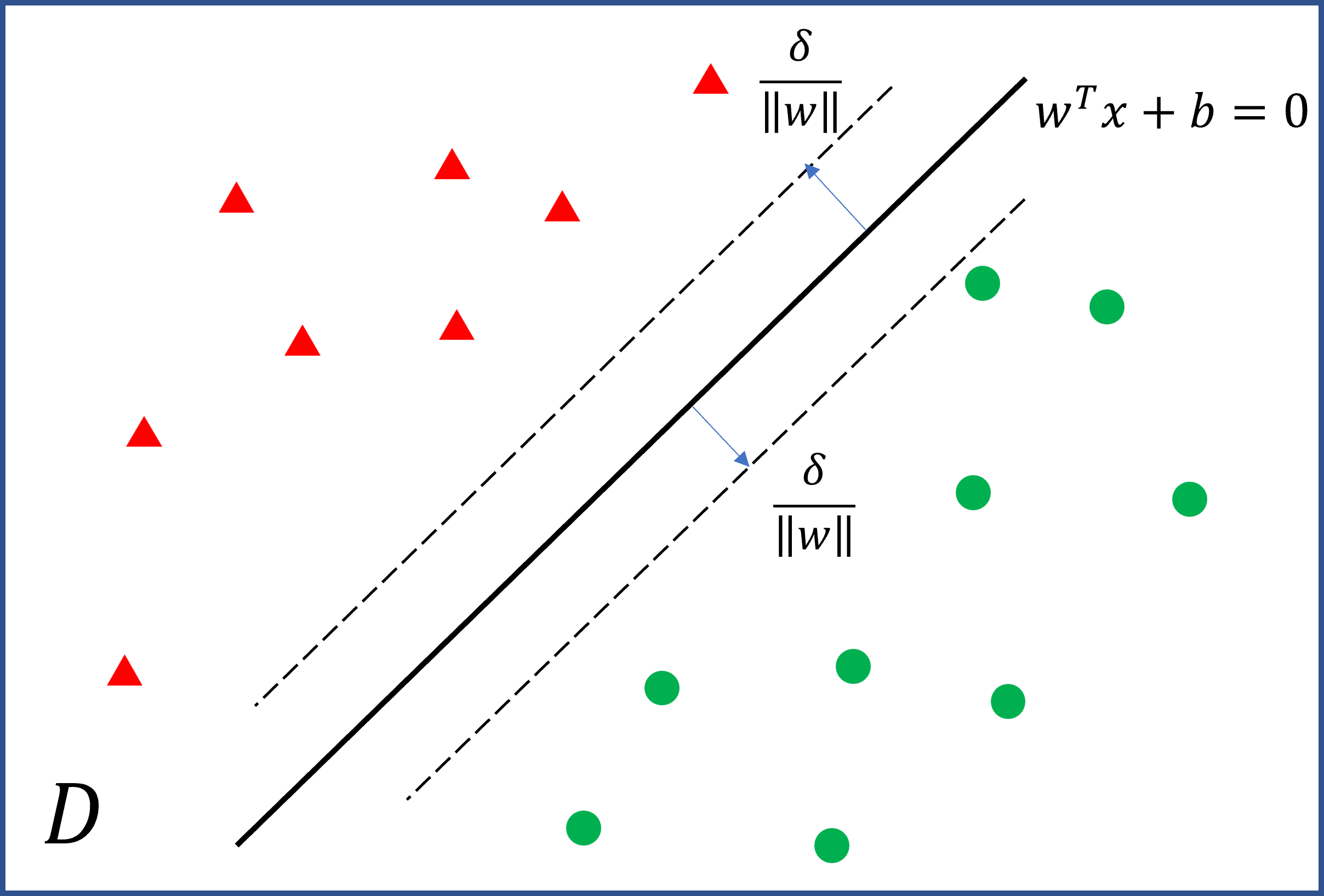}
\caption{The effective region $\mathcal{E}(w, b)$.}
\end{wrapfigure}

Suppose the sample domain $D$ is partitioned into equal grids of length $\delta > 0$, the gradient of the MSE $\mathcal{L}(\theta)$ is actually estimated by
$$\nabla \mathcal{L}(\theta) \simeq \frac{1}{N} \sum_{i = 1}^N \frac{\partial L}{\partial \theta} (x_i, \theta)$$
which means that the parameters will stop updating when the effective region is too small to contain any data points. 

For instance, consider the approximation of the target function $\phi = \chi_{[0, 1]}$ over the domain $D = [-1, 1]$ using a single neuron $f(x; w) = \sigma(wx)$, so that the global minimum is at $w = \infty$ with zero loss. Now consider the partition $x_i = -1 + \frac{i-1}{K}$ for $i = 1, 2, ..., 2K + 1$, according to the previous results, the largest reachable value of $\| w \|$ by gradient descent cannot exceed $\frac{(p-1)\log 2}{\Delta x}$ where $\Delta x = \frac{1}{K}$ is the grid size. 

On the other hand, the actual $L^2$ error at $w' = \frac{(p-1)\log 2}{\Delta x}$ is
\begin{equation}
\|  f(\cdot; w') - \phi \|_2 = \sqrt{\frac{2}{w'} (\log(\frac{2}{1 + e^{-w'}}) + 1 - \frac{2}{1 + e^{w'}}) } 
\sim \mathcal{O}( \sqrt{\Delta x} \ |\log \epsilon|^{-\frac{1}{2}})
\end{equation}
when $|w'|\gg 1$, which then leads to Theorem~\ref{th:minimum}.
\end{proof}

\begin{wrapfigure}{r}{0.35\textwidth}
\vspace{-12pt}
\centering
\includegraphics[width=0.35\textwidth]{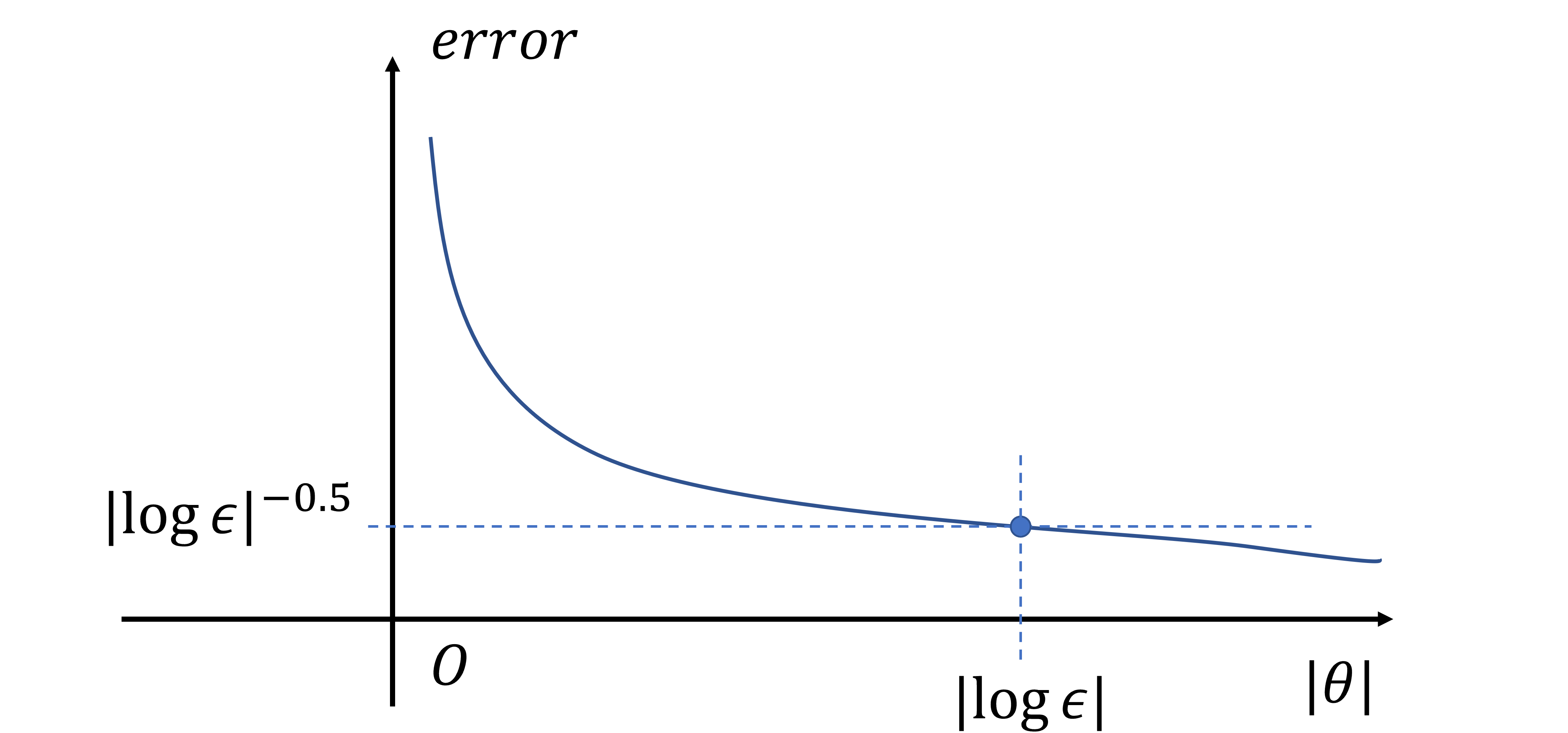}
\caption{The minimal achievable error is determined by $|\log \epsilon|$.}
\label{fig:log-eps}
\end{wrapfigure}

Theorem~\ref{th:minimum} suggests that the unreachable infinity problem is hard to address for two reasons: first, increasing the machine precision does not significantly improve the performance, as the minimum of error is proportional to $|\log \epsilon|^{-\frac{1}{2}}$ which decays very slowly (Figure \ref{fig:log-eps}); second, The discretization error is a fundamental limitation present in both PINNs and traditional methods, suffering from the curse of dimensionality \citep{aubin}. Additional discussions on numerical issues when the global minimum does not exist can also be found in \cite{glorot} and \cite{gallon}.

\section{Experiments}
\label{sec:exp}

In this section, we validate our theoretical results in the one-dimensional Burgers' equation \citep{basdevant1986spectral}:
\begin{equation}
\label{burgers}
    \begin{cases}
        & \frac{\partial u}{\partial t} + \mu u \frac{\partial u}{\partial x} = \nu \frac{\partial^2 u}{\partial x^2}, \quad  x \in [-1, 1], t \in [0, 1];\\
      & u(x, 0) = \sin\left(\frac{\pi x}{2}\right), \quad x \in \mathbb{R};  \\
    \end{cases} 
\end{equation}
where $\mu=-1, \nu=10^{-3}$. The Burgers' equation is a nonlinear partial differential equation that models a combination of advection and diffusion. This system is carefully selected because both of its initial condition and differential operator are smooth, but they finally lead to a discontinuous solution at $t = 1$. We first solve \eqref{burgers} by minimizing its PINN objective. Then, we train on a new loss function that combines the $L^2$ residual with data from the true solution.

\paragraph{Solving with PINN Loss.}

The solution is represented by a 2-hidden-layer \textit{sigmoid} network whose size is $256 \times 256$. We minimize the loss function \eqref{pdeloss} via SGD and Adam separately, and the results are presented in Figure~\ref{fig:burgers}. It turns out that both optimizers provide smooth approximations that do not match the true solution.

\begin{figure}[t!]
\centering

\subfigure[]{\includegraphics[width=0.24\textwidth]{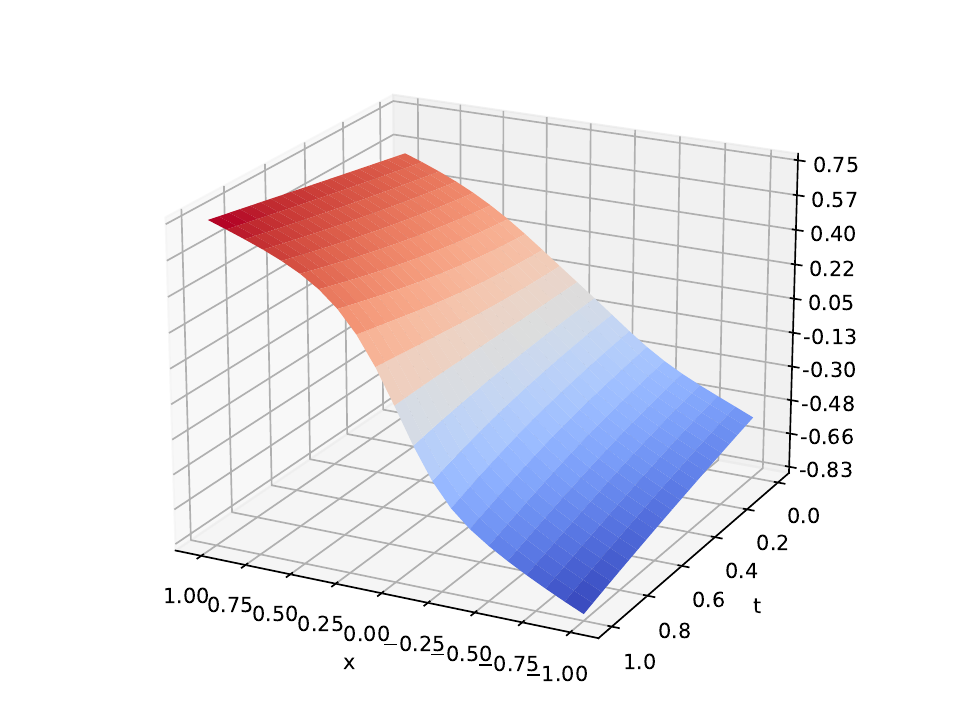}}
\subfigure[]{\includegraphics[width=0.24\textwidth]{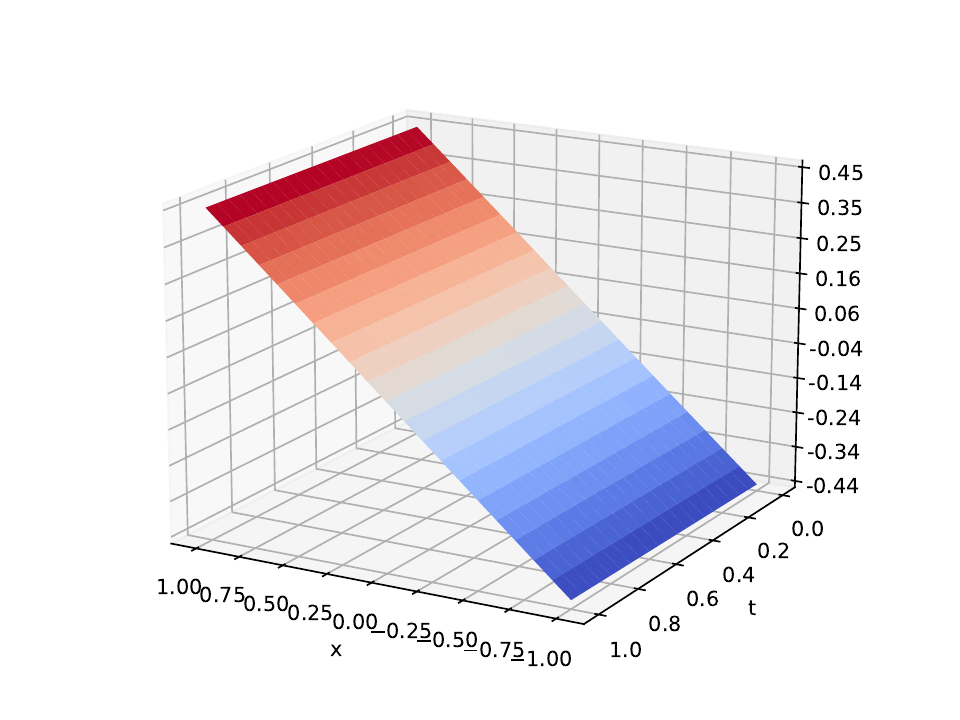}}
\subfigure[]{\includegraphics[width=0.24\textwidth]{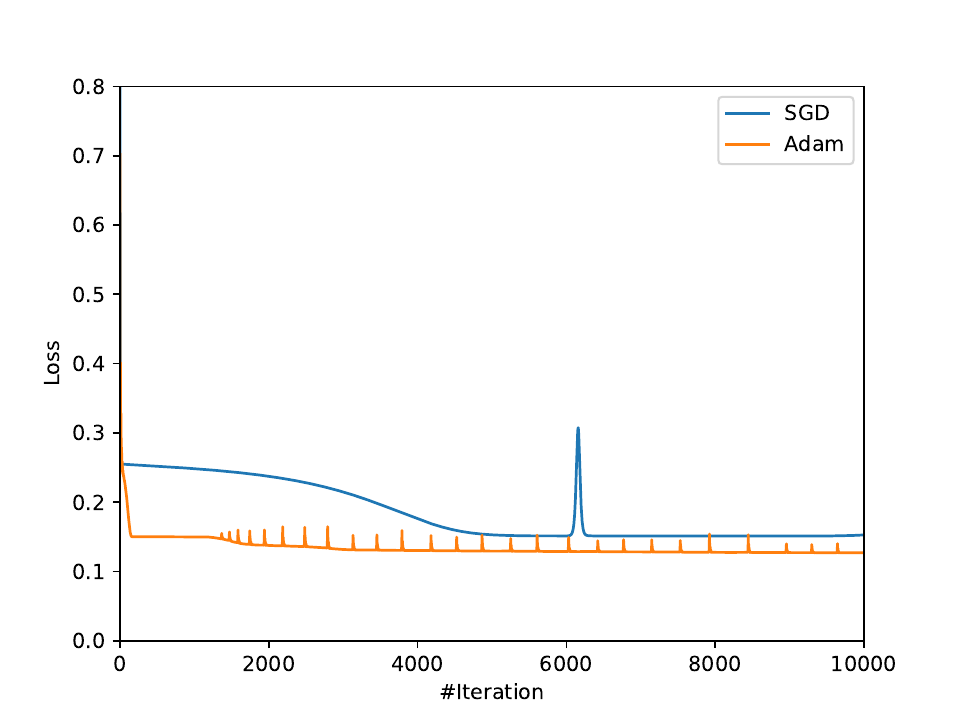}}
\subfigure[]{\includegraphics[width=0.24\textwidth]{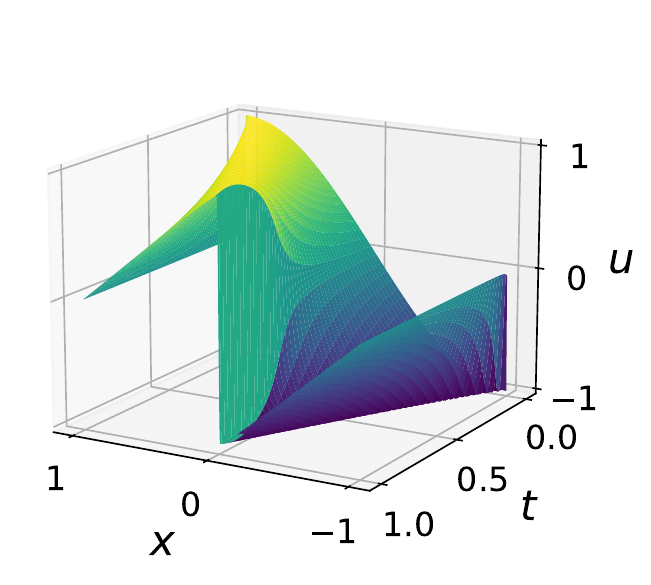}}

\caption{(a) Solution by Adam; (b) Solution by SGD; (c) Training curves; (d) True solution of \eqref{burgers}.}
\label{fig:burgers}.
\end{figure}

\paragraph{Approximation gap of neural networks.}

To examine the limitation of neural network approximation for discontinuous solutions, we minimize the new loss function
$$\tilde{J}(\theta) =  \frac{1}{K} \sum_{i = 1}^K |u(x_i, t_i; \theta) - u_i|^2,$$
where $u_i$ are the values of true solution at $(x_i, t_i)$, which are collected by computing the numerical solution using the spectral method implemented in \cite{binder2021}. The time resolution is 0.01 and spatial resolution is 0.001. Figure \ref{fig:burgers} (d) visualizes the solution. We evaluate the mean squared error of fully connected networks with \textit{sigmoid} activation trained under SGD. In the first set of experiment, we use a 2 layer network with hidden size ranging from 2 to 256. In the second set, we vary the number of layers while keeping all hidden dimension at 4. Figure \ref{fig:burgers_loss_width} (b,c) show that a better approximation is obtained than using $L^2$ residual only. They also show the approximation error of neural networks with different width and depth.
Figure \ref{fig:burgers_loss_2d} visualizes the predicted solution against the numerical solution at different $t$.

\begin{figure}[h!]
\centering

\subfigure[]{\includegraphics[width=0.24\textwidth]{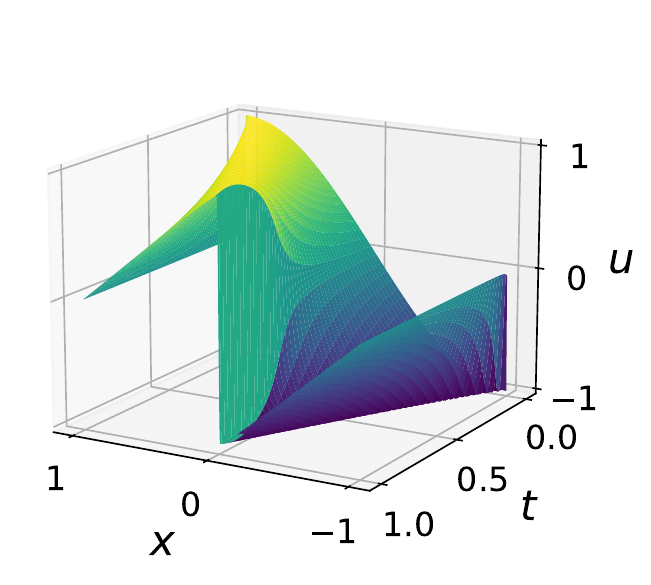}}
\subfigure[]{\includegraphics[width=0.24\textwidth]{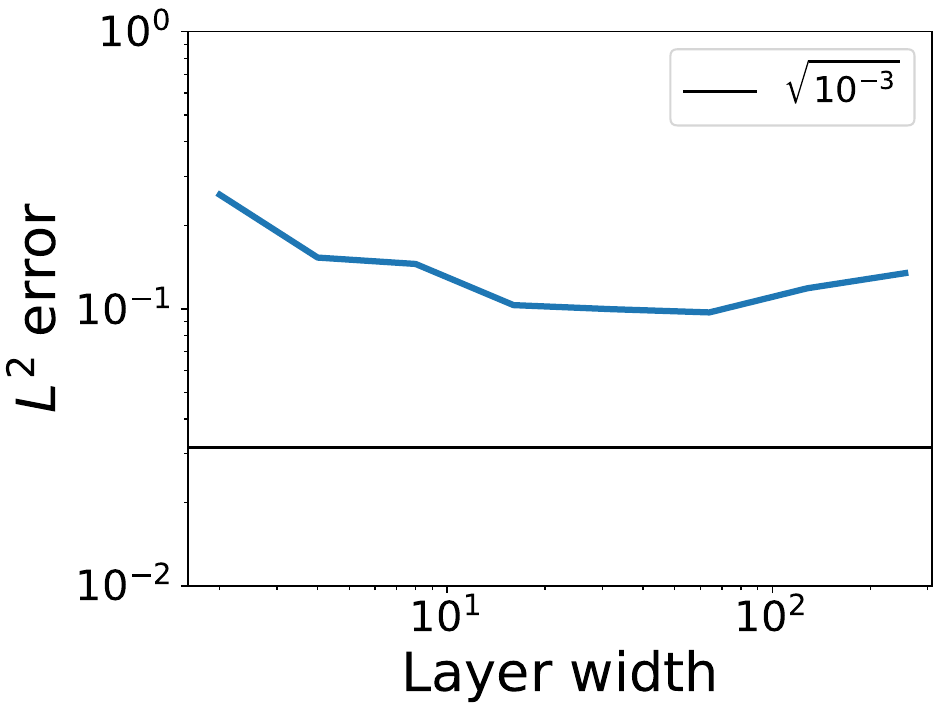}}
\subfigure[]{\includegraphics[width=0.24\textwidth]{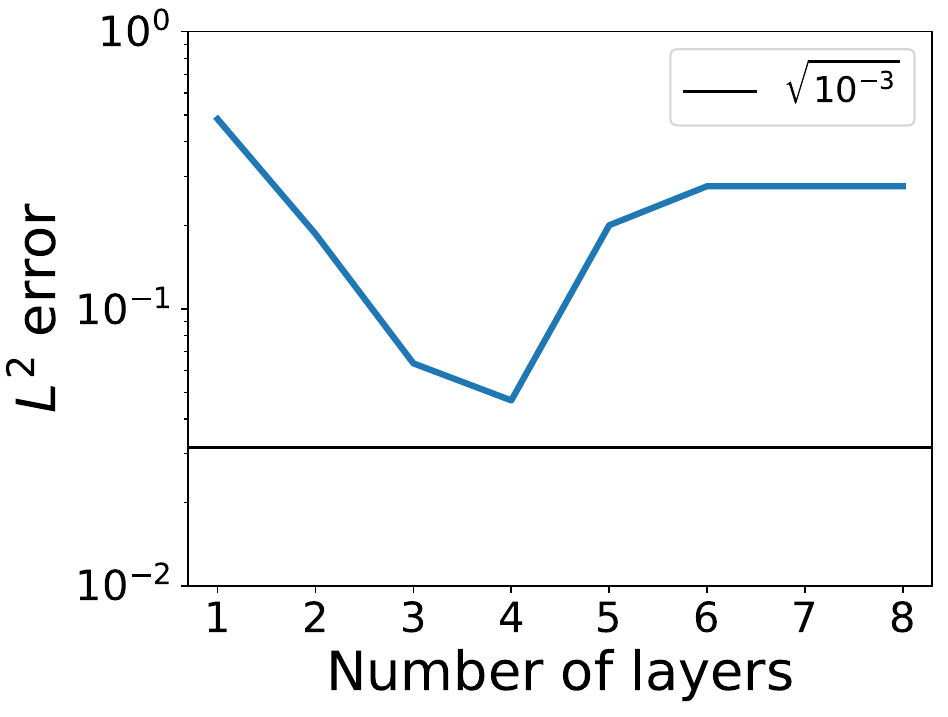}}

\caption{(a) Numerical solution represented by neural network; (b) Neural network approximation error vs width; (c) Neural network approximation error vs number of layers.}
\label{fig:burgers_loss_width}
\end{figure}

\begin{figure}[h!]
\centering
\subfigure[]{\includegraphics[width=0.19\textwidth]{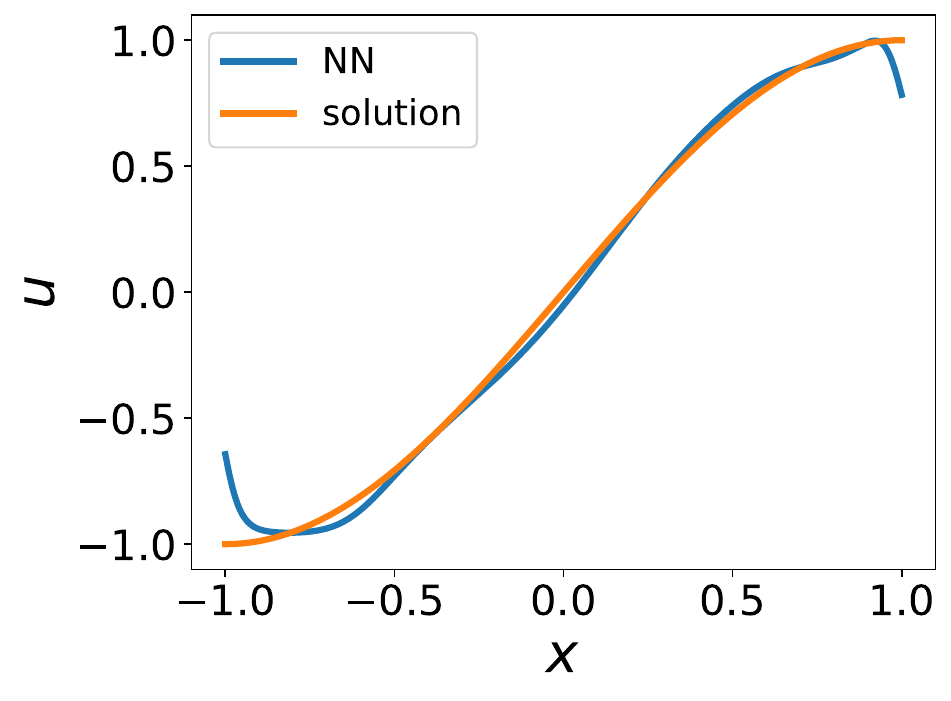}}
\subfigure[]{\includegraphics[width=0.19\textwidth]{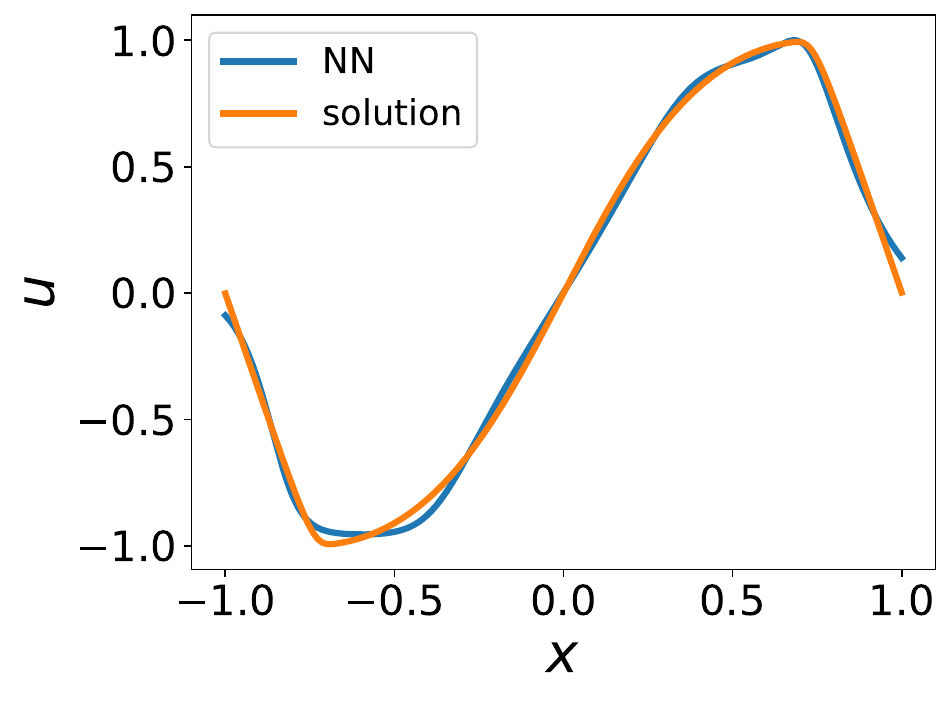}}
\subfigure[]{\includegraphics[width=0.19\textwidth]{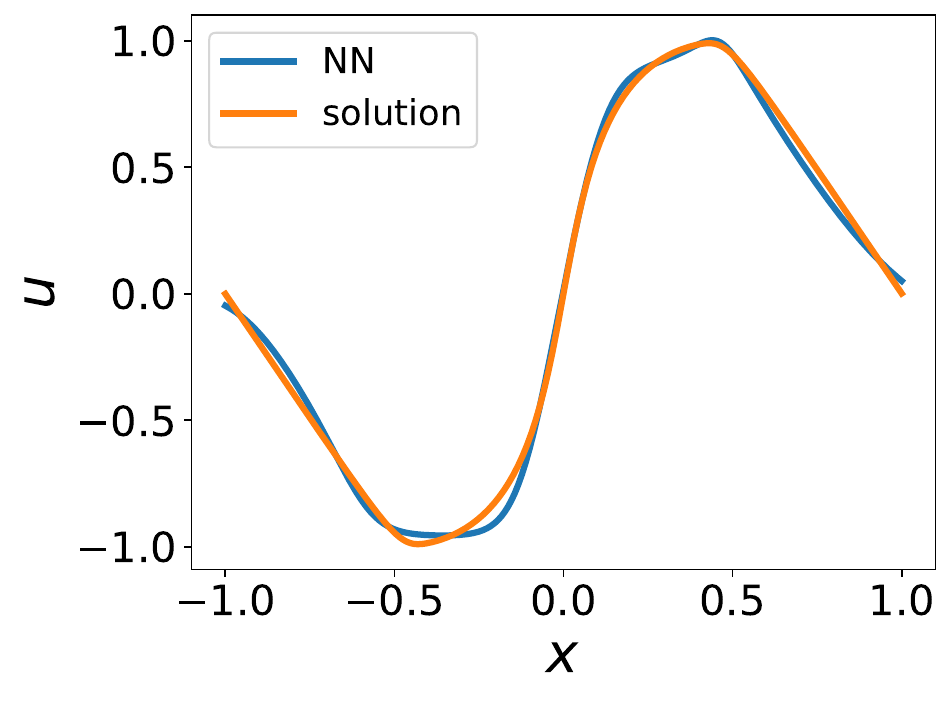}}
\subfigure[]{\includegraphics[width=0.19\textwidth]{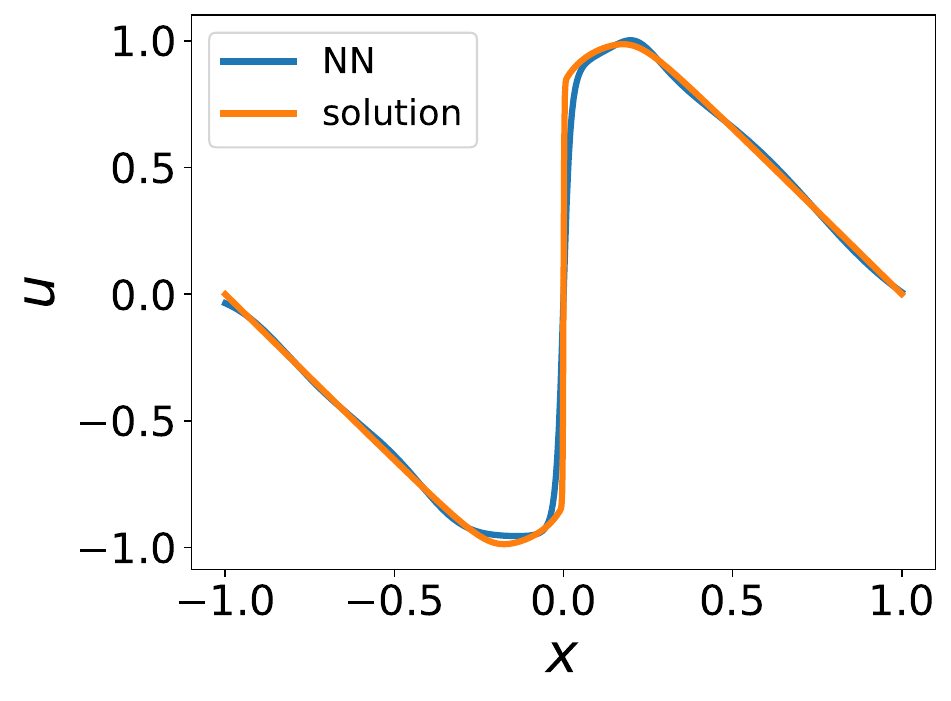}}
\subfigure[]{\includegraphics[width=0.19\textwidth]{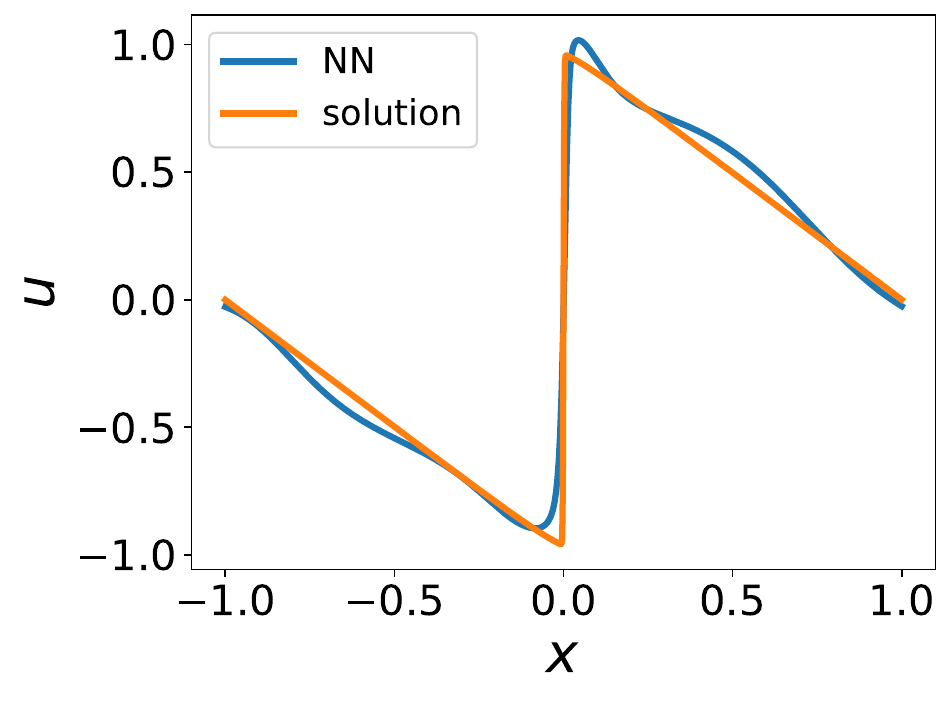}}

\caption{Comparison of predicted and analytical solution of Burgers' equation at different time: (a) $t = 0$; (b) $t = 0.25$; (c) $t = 0.5$; (d) $t = 0.75$; (e) $t = 1$. A shock wave is formed at $t = 1$.}
\label{fig:burgers_loss_2d}
\end{figure}

\paragraph{Analysis.} 
Figure \ref{fig:burgers_loss_width} (a) shows that it is possible to use neural networks to represent the true solution. This verifies that the failure case in Figure~\ref{fig:burgers} is due to the limitation of PINN loss \eqref{pdeloss} in capturing the emergence of shock waves as suggested in Section~\ref{sec:residual}. On the other hand, even with the real data, the best achieved $L^2$ error still exceeds the square root of the resolution ($\delta = 10^{-3}$), which further validates the claim on machine precision in Theorem~\ref{th:minimum}.

\section{Conclusion}

In this work, we discuss the fundamental limitations of physics-informed neural networks (PINNs) in solving Cauchy problems. However, these issues rarely affect classical methods, prompting a reconsideration of using neural networks for scientific computing. While we focus on the aspects of $L^2$ residual and neural network approximation, it is worth noting that various factors, such as the sampling scheme and optimizer choice, can also influence the numerical performance. Overall, we advocate for a comprehensive understanding of the foundation of dynamical models for future integration of deep learning in scientific computing.

\section*{Acknowledgments}

This work was supported in part by Army-ECASE award W911NF-23-1-0231, the U.S. Department Of Energy, Office of Science under \#DE-SC0022255, IARPA HAYSTAC Program, CDC-RFA-FT-23-0069, NSF Grants \#2205093, \#2146343, and \#2134274.

We would like to thank the anonymous reviewers for their valuable suggestions.

\bibliography{references}

\end{document}